\documentclass[runningheads]{llncs}
\pdfoutput=1

\usepackage{color}
\usepackage{times}
\usepackage{latexsym} 
\usepackage{amssymb} 
\usepackage{url}
\usepackage{xspace}
\usepackage{graphicx}
\usepackage{tikz}
\usetikzlibrary{arrows,fit,shapes}
\usepackage{comment}

\newcommand{\nop}[1]{}
\newcommand{\equationlabel}[1]{\refstepcounter{equation}\label{#1}(\arabic{equation})}

\newcounter{structure}
\newcommand{\structurelabel}[1]{{\refstepcounter{structure}\label{#1}\arabic{structure}}}

\newcounter{bijection}
\newcommand{\bijectionlabel}[1]{{\refstepcounter{bijection}\label{#1}\arabic{bijection}}}

\newcounter{program}
\newcommand{\programlabel}[1]{{\refstepcounter{program}\label{#1}\arabic{program}}}

\newcommand{\capminus}{\mathrel{\ooalign{\hidewidth$\cap$\hidewidth\cr\scalebox{0.6}[1.0]{$-$}}}}
\newcommand{\BP}{\ensuremath{\mathcal{U}}\xspace}

\newcommand{\coNP}{\ensuremath{\mathit{coNP}}\xspace}

\newcommand{\Angstrom}{\mbox{\textit{\AA}}}

\title{Properties of Answer Set Programming with Convex Generalized Atoms}
\author{Mario Alviano \and Wolfgang Faber}

\institute{Department of Mathematics and Computer Science\\
           University of Calabria\\
           87036 Rende (CS), Italy \\
           \email{\{alviano,faber\}@mat.unical.it}
}

\titlerunning{Properties of Answer Set Programming with Convex Generalized Atoms}
\authorrunning{M.~Alviano and W.~Faber} 
\begin{document}
\maketitle

\begin{abstract}
In recent years, Answer Set Programming (ASP), logic programming under
the stable model or answer set semantics, has seen several extensions
by generalizing the notion of an atom in these programs: be it
aggregate atoms, HEX atoms, generalized quantifiers, or abstract
constraints, the idea is to have more complicated satisfaction
patterns in the lattice of Herbrand interpretations than traditional,
simple atoms. In this paper we refer to any of these constructs as
generalized atoms. Several semantics with differing characteristics
have been proposed for these extensions, rendering the big picture
somewhat blurry. In this paper, we analyze the class of programs that
have convex generalized atoms (originally proposed by Liu and
Truszczy{\'n}ski in \cite{liu-trus-2006-jair}) in rule bodies and show
that for this class many of the proposed semantics coincide. This is
an interesting result, since recently it has been shown that this
class is the precise complexity boundary for the FLP semantics. We
investigate whether similar results also hold for other semantics, and
discuss the implications of our findings.
\end{abstract}

\section{Introduction}

Various extensions of the basic Answer Set Programming language have
been proposed by allowing more general atoms in rule bodies, for
example aggregate atoms, HEX atoms, dl-atoms, generalized quantifiers,
or abstract constraints. A number of semantics have been proposed for
such programs, most notably the FLP semantics
\cite{fabe-etal-2011-aij} and a number of coinciding semantics that we
will collectively refer to as PSP semantics
(from Pelov, Son, and Pontelli) \cite{pelo-2004,son-pont-2007-tplp}. All of these semantics coincide
with traditional ASP semantics when no generalized atoms are
present. Moreover, they coincide on programs that have atomic rule
heads and contain only monotonic generalized atoms. In
\cite{liu-etal-2010-aij} it is furthermore hinted that the semantics
also coincide on programs that have atomic rule heads and contain only
convex generalized atoms. However, no formal proof is available for this
claim, and the informal explanation given in \cite{liu-etal-2010-aij}
is not as general as it could be, as we will show.

In this paper, we undertake a deeper investigation on the similarities
and differences between the FLP and PSP semantics. In order to do
this, we consider a simplified, yet expressive propositional language:
sets of rules with atomic heads and bodies that are formed of a single
``structure,'' which are functions mapping interpretations to Boolean
values\footnote{Note that (apart from the name) there is no connection to structures in first-order logic.}. Clearly, structures encompass atoms, literals, and
conjunctions thereof, but can represent any propositional formula,
generalized atom, or conjunctions of generalized atoms. Each structure
has an associated domain, which is the set of propositional atoms on
which the structure's truth valuation depends. We can then classify
the structures by their semantic properties, in particular, we will
focus on the class of convex structures, which have single contiguous
areas of truth in the lattice of interpretations. Convex structures
include atoms and literals, and they are closed under conjunction (but
not under negation or disjunction).

We first formally prove the claim that the FLP and PSP semantics
coincide on programs with convex structures, as originally reported in
\cite{liu-etal-2010-aij}. We will then move on to the main focus of
this paper, trying to understand whether there is any larger class for
which the semantics coincide. It is known that for programs with
general structures all PSP answer sets are FLP answer sets, but not
all FLP answer sets are PSP answer sets. The precise boundary for
exhibiting the semantic difference is instead unknown.

We will approach this question using complexity arguments. Recently,
we could show that convex structures form the precise boundary for a
complexity jump in the polynomial hierarchy on cautious reasoning (but
most other decision problems as well) for the FLP semantics. Cautious
reasoning is $\Pi^P_2$-complete for the FLP semantics when allowing
any non-convex structure and its variants (renaming atoms) in the
input program, but it is $\coNP$-complete for convex structures. When
considering the PSP semantics, cautious reasoning is also
$\Pi^P_2$-complete when allowing any kind of structures in the
input. This follows from a result in \cite{pelo-2004}, and we provide
an alternative proof in this paper. Analyzing this proof, it becomes
clear that there is a different source of complexity for PSP than for
FLP.

We then show that this different source of complexity also yields a
different shape of the boundary for the complexity jump in
PSP. Indeed, we first show that for a simple non-convex structure,
cautious reasoning is still in $\coNP$ for the PSP semantics, while
the problem is $\Pi^P_2$-hard in the presence of this structure for
the FLP semantics. It turns out that the same argument works for many
non-convex structures, in particular, for all structures with a domain
size bounded by a constant. The domain size therefore serves as a
parameter that simplifies the complexity of the problem for the PSP
semantics (unless the polynomial hierarchy collapses to its first
level). This also means that the complexity boundary for PSP has a
non-uniform shape, in the sense that an infinite number of different
non-convex structures must be available for obtaining
$\Pi^P_2$-hardness for cautious reasoning. This is in contrast to the
FLP semantics, where the presence of a single non-convex structure is
sufficient.

\section{Syntax and Semantics} \label{sec:syntax}

In this section we first introduce the syntax used in the paper.
This is mainly based on the notion of structures, i.e., functions mapping interpretations into Boolean truth values.
Then, we introduce few semantic notions and in particular we characterize structures in terms of monotonicity.
Finally, we define the two semantics analyzed in this paper, namely FLP and PSP.

\subsection{Syntax}

Let $\BP$ be a fixed, countable set of propositional atoms.
An interpretation $I$ is a subset of $\BP$.
A structure $S$ on $\BP$ is a mapping of interpretations into Boolean truth values.
Each structure $S$ has an associated, finite domain $D_S \subset \BP$, indicating those atoms that are relevant to the structure.

\begin{example}\label{ex:structures}
A structure $S_\structurelabel{structure:1}$ modeling a conjunction $a_1, \ldots, a_n$ ($n \geq 0$) of propositional atoms is such that $D_{S_{\ref{structure:1}}} = \{a_1, \ldots, a_n\}$ and, for every interpretation $I$, $S_{\ref{structure:1}}$ maps $I$ to true if and only if $D_{S_{\ref{structure:1}}} \subseteq I$.

A structure $S_\structurelabel{structure:2}$ modeling a conjunction $a_1, \ldots, a_m, not\ a_{m+1}, \ldots, not\ a_n$ ($n \geq m \geq 0$) of literals, where $a_1, \ldots, a_n$ are propositional atoms and $not$ denotes \emph{negation as failure}, is such that $D_{S_{\ref{structure:2}}} = \{a_1, \ldots, a_n\}$ and, for every interpretation $I$, $S_{\ref{structure:2}}$ maps $I$ to true if and only if $\{a_1, \ldots, a_m\} \subseteq I$ and $\{a_{m+1}, \ldots, a_n\} \cap I = \emptyset$.

A structure $S_\structurelabel{structure:3}$ modeling an aggregate $\mathit{COUNT}(\{a_1, \ldots, a_n\}) \neq k$ ($n \geq k \geq 0$), where $a_1, \ldots, a_n$ are propositional atoms, is such that $D_{S_{\ref{structure:3}}} = \{a_1, \ldots, a_n\}$ and, for every interpretation $I$, $S_{\ref{structure:3}}$ maps $I$ to true if and only if $|D_{S_{\ref{structure:3}}} \cap I| \neq k$.
\end{example}

A general rule $r$ is of the following form:
\begin{equation}\label{eq:rule}
H(r) \leftarrow B(r)
\end{equation}
where $H(r)$ is a propositional atom in $\BP$ referred as the head of
$r$, and $B(r)$ is a structure on $\BP$ called the body of $r$.
A general program $P$ is a set of general rules.

\begin{example}\label{ex:program}
Let $S_{\structurelabel{structure:4}}$ map to true any interpretation $I$ such that $I \cap \{a,b\} \neq \{b\}$,
and let $S_{\structurelabel{structure:5}}$ map to true any interpretation $I$ such that $I \cap \{a,b\} \neq \{a\}$.
Hence, program
$P_{\programlabel{program:1}} = \{a \leftarrow S_{\ref{structure:4}}; b \leftarrow S_{\ref{structure:5}}\}$ 
is equivalent to the following program with aggregates:
\begin{eqnarray*}
a & \leftarrow & \mathit{SUM}(\{a = 1, b = -1\}) \geq 0\\
b & \leftarrow & \mathit{SUM}(\{a = -1, b = 1\}) \geq 0
\end{eqnarray*}
\end{example}

Note that no
particular assumption is made on the syntax of rule bodies; in the case of
normal propositional logic programs these structures are conjunctions
of literals. We assume that structures are closed under propositional
variants, that is, if $S$ is a structure, for any bijection $\sigma: \BP \rightarrow \BP$,
also $S\sigma$ is a structure, and the
associated domain is $D_{S\sigma} = \{\sigma(a) \mid a \in D_S\}$.

\begin{example}\label{ex:variants}
Consider $S_{\ref{structure:4}}$ and $S_{\ref{structure:5}}$ from Example~\ref{ex:program}, and a bijection $\sigma_\bijectionlabel{bijection:1}$ such that $\sigma_{\ref{bijection:1}}(a) = b$.
Hence, $S_{\ref{structure:5}} = S_{\ref{structure:4}}\sigma_{\ref{bijection:1}}$, that is, $S_{\ref{structure:5}}$ is a variant of $S_{\ref{structure:4}}$.
%
\end{example}

Given a set of structures $\mathbf{S}$, by datalog$^\mathbf{S}$ we
refer to the class of programs that may contain only the following rule bodies:
structures corresponding to conjunctions of atoms, any structure $S \in \mathbf{S}$, or any
of its variants $S\sigma$.

\begin{example}\label{ex:classes}
For every $n \geq m \geq 0$, let $S^{m,n}$ denote the structure $S_{\ref{structure:2}}$ from Example~\ref{ex:structures}.
The class of normal datalog programs is datalog$^{\{S^{m,n} \mid n \geq m \geq 0\}}$.
\end{example}

Note that this syntax does not explicitly allow for negated structures. One can, however, choose the complementary structure for simulating negation. This would be akin to the ``negation as complement'' interpretation of negated aggregates that is prevalent in the literature.

\subsection{Semantics}

Let $I \subseteq \BP$ be an interpretation.
$I$ is a model for a structure $S$, denoted $I \models S$, if $S$ maps $I$ to true.
Otherwise, if $S$ maps $I$ to false, $I$ is not a model of $S$, denoted $I \not\models S$. We require that atoms outside the domain of $S$ are irrelevant for modelhood, that is, for any interpretation $I$ and $X \subseteq \BP \setminus D_S$ it holds that $I \models S$ if and only if $I\cup X \models S$. Moreover, for any bijection $\sigma: \BP \rightarrow \BP$, let $I\sigma = \{\sigma(a) \mid a \in I\}$, and we require that $I\sigma \models S\sigma$ if and only if $I \models S$.
$I$ is a model of a rule $r$ of the form (\ref{eq:rule}), denoted $I \models r$, if $H(r) \in I$ whenever $I \models B(r)$.
$I$ is a model of a program $P$, denoted $I \models P$, if $I \models r$ for every rule $r \in P$.

\begin{example}
Consider program $P_{\ref{program:1}}$ from Example~\ref{ex:program}.
It can be observed that
$\emptyset \not\models P_{\ref{program:1}}$ and $\{a,b\} \not\models P_{\ref{program:1}}$ (both rules have true bodies but false heads),
while $\{a\} \models P_{\ref{program:1}}$ and
$\{b\} \models P_{\ref{program:1}}$.
\end{example}


Structures can be characterized in terms of \emph{monotonicity} as follows.
\begin{definition}[Monotone Structures]
A structure $S$ is monotonic if for all pairs $X,Y$ of interpretations such that $X \subset Y$, $X \models S$ implies $Y \models S$.
\end{definition}
\begin{definition}[Antimonotone Structures]
A structure $S$ is antimonotonic if for all pairs $Y,Z$ of interpretations such that $Y \subset Z$, $Z \models S$ implies $Y \models S$.
\end{definition}
\begin{definition}[Convex Structures]
A structure $S$ is convex if for all triples $X,Y,Z$ of interpretations such that $X \subset Y \subset Z$, $X \models S$ and $Z \models S$ implies $Y \models S$.
\end{definition}
Note that monotonic and antimonotonic structures are convex.
Moreover, note that convex structures are closed under conjunction (but not under disjunction or negation).

\begin{example}
Structure $S^{m,n}$ from Example~\ref{ex:classes} is convex in general;
it is monotonic if $m = n$, and antimonotonic if $m = 0$.
Structure $S_{\ref{structure:3}}$ from Example~\ref{ex:structures}, instead, is non-convex if $n > k > 0$;
it is monotonic if $k = 0$, and antimonotonic if $n = k$.
\end{example}

We first describe a reduct-based semantics, usually referred to as FLP, which has been described and analyzed in \cite{fabe-etal-2004-jelia,fabe-etal-2011-aij}.

\begin{definition}[FLP Reduct]
The FLP reduct $P^I$ of a program $P$ with respect to $I$ is defined as the set $\{r \in P \mid I \models B(r)\}$.
\end{definition}

\begin{definition}[FLP Answer Sets]
$I$ is an FLP answer set of $P$ if $I \models P^I$ and for each $J \subset I$ it holds that $J \not\models P^I$.
\end{definition}

\begin{example}
Consider program $P_{\ref{program:1}}$ from Example~\ref{ex:program} and the interpretation $\{a\}$.
The reduct $P_{\ref{program:1}}^{\{a\}}$ is $\{a \leftarrow S_{\ref{structure:4}}\}$.
Since $\{a\}$ is a minimal model of the reduct, $\{a\}$ is an FLP answer set of $P_{\ref{program:1}}$.
Similarly, it can be observed that $\{b\}$ is another FLP answer set.
Actually, these are the only FLP answer sets of the program.
\end{example}

We will next describe a different semantics, using the definition of \cite{son-pont-2007-tplp}, called ``fixpoint answer set'' in that paper. Theorem 3 in \cite{son-pont-2007-tplp} shows that it is actually equivalent to the two-valued fix-point of ultimate approximations of generalized atoms in \cite{pelo-2004}\footnote{There is an even closer relationship, as the operator $K^P_M(I)$ of \cite{son-pont-2007-tplp} coincides with $\phi^{aggr,1}_P(I,M)$ defined in \cite{pelo-2004}, as shown in the appendix of \cite{son-pont-2007-tplp}}, and therefore with stable models for ultimate approximations of aggregates as defined in \cite{pelo-etal-2007-tplp}. We will refer to it as PSP to abbreviate Pelov/Son/Pontelli, the names most frequently associated with this semantics.

\begin{definition}[Conditional Satisfaction]
A structure $S$ on $\BP$ is \emph{conditionally satisfied} by a pair of interpretations $(I,M)$, denoted $(I,M) \models S$, if $J \models S$ for each $J$ such that $I \subseteq J \subseteq M$.
\end{definition}

\begin{definition}[PSP Answer Sets]
An interpretation $M$ is a PSP answer set if $M$ is the least fixpoint of the following operator:
\begin{equation}
K^P_M(I) = \{H(r) \mid r \in P \wedge (I,M) \models B(r)\}.
\end{equation}
\end{definition}

\begin{example}
Consider program $P_{\ref{program:1}}$ from Example~\ref{ex:program} and the interpretation $\{a\}$.
The least fixpoint of $K^{P_{\ref{program:1}}}_{\{a\}}$ is $\{a\}$.
In fact, $\emptyset \models S_{\ref{structure:4}}$ and $\{a\} \models S_{\ref{structure:4}}$, hence $(\emptyset, \{a\}) \models S_{\ref{structure:4}}$, while $\{a\} \not\models S_{\ref{structure:5}}$ and thus $(\emptyset, \{a\}) \not\models S_{\ref{structure:5}}$ and $(\{a\}, \{a\}) \not\models S_{\ref{structure:5}}$.
Therefore, $\{a\}$ is a PSP answer set.
Also $\{b\}$ is a PSP answer set.
\end{example}

On programs considered in this paper, PSP answer sets also coincide with ``answer sets'' defined in \cite{son-etal-2007-jair} (by virtue of Proposition 10 in \cite{son-etal-2007-jair}) and ``well-justified FLP answer sets'' of \cite{shen-wang-2012-aaai} (by virtue of Theorem 5 in \cite{shen-wang-2012-aaai}). The latter is particularly interesting, as it is defined by first forming the FLP reduct.
Indeed, as shown in \cite{shen-wang-2012-aaai}, the operator $K^P_M$ can be equivalently defined as follows:
\begin{equation}
K^P_M(I) = \{H(r) \mid r \in P^M \wedge (I,M) \models B(r)\}.
\end{equation}

There are several other semantic definitions on programs that have some restrictions on the admissible structures, which also coincide with the PSP semantics on programs as defined in this paper with the respective structure restriction. Examples are \cite{mare-etal-2004-lpnmr} for monotonic structures (that are also allowed to occur in rule heads in that paper), or \cite{simo-etal-2002} that allows for structures corresponding to cardinality and weight constraints and largely coincide with the PSP semantics (see \cite{pelo-2004} for a discussion on structures on which the semantics coincides).

In this paper we are mainly interested in cautious reasoning, defined next.

\begin{definition}[Cautious Reasoning]
A propositional atom $a$ is a cautious consequence of a program $P$ under FLP (resp.\ PSP) semantics, denoted $P \models^{\mathit{FLP}}_c a$ (resp.\ $P \models^{\mathit{PSP}}_c a$), if $a$ belongs to all FLP (resp.\ PSP) answer set of $P$.
\end{definition}

\begin{example}
Consider program $P_{\ref{program:1}}$ from Example~\ref{ex:program}.
We have $P_{\ref{program:1}} \not\models^{\mathit{FLP}}_c a$ and $P_{\ref{program:1}} \not\models^{\mathit{FLP}}_c b$, and similar for PSP semantics.
If we add $a \leftarrow b$ and $b \leftarrow a$ to the program, then there is only one FLP answer set, namely $\{a,b\}$, and no PSP answer sets.
In this case $a$ and $b$ are cautious consequences of the program (under both semantics).
\end{example}

\section{Exploring the Relationship between the FLP and PSP Semantics}

In this section, we examine in detail how the FLP and PSP semantics relate. We shall proceed in three steps. First, we formally prove that FLP and PSP semantics coincide on programs with convex structures in Section~\ref{sec:coincidence}. Next, we turn towards complexity as a tool to understand whether there can be any larger class of coinciding programs. We start in Section~\ref{sec:consonance} with a result that shows that programs without restrictions exhibit the same complexity under both FLP and PSP semantics. However, it is known that the semantics do not coincide for programs without restrictions, and we examine the complexity proofs to highlight the different complexity sources. These findings are then applied in Section~\ref{sec:dissonance} in order to identify programs with bounded non-convex structures, on which the complexities for FLP and PSP semantics differ. Under usual complexity assumptions, this also implies that programs with convex aggregates is the largest class of programs on which FLP and PSP coincide.

\subsection{Unison: Convex Structures} \label{sec:coincidence}

In this section we show that for programs with convex aggregates the FLP and PSP semantics coincide.  In \cite{liu-etal-2010-aij} it is stated that many semantics (and in particular, FLP and PSP) ``agree on [...] programs with convex aggregates'' because ``they can be regarded as special programs with monotone constraints.'' However, the comment on regarding convex aggregates as monotone constraints relies on a transformation described in \cite{liu-trus-2006-jair} that transforms convex structures into conjunctions of positive and negated monotone constraints. Since our language does not explicitly allow negation, and in particular since convex structures are not closed under negation, we next prove in a more direct manner that the FLP and PSP semantics coincide on convex structures.

One direction of the proof relies on the well-known more general fact that each PSP answer set is also an FLP answer set. This has been stated as Theorem~2 in \cite{son-pont-2007-tplp} and Proposition~8.1 in \cite{pelo-etal-2007-tplp}.

\begin{theorem}
Let $P$ be program whose body structures are convex, and let $M$ be an interpretation.
$M$ is an FLP answer set of $P$ if and only if $M$ is an PSP answer set of $P$.
\end{theorem}
\begin{proof}
The left implication follows from Theorem~2 in \cite{son-pont-2007-tplp}.
For the right implication, let $M$ be an FLP answer set of $P$.
Let $K_0 := \emptyset$, $K_{i+1} := K^P_M(K_i)$ for $i \geq 0$, and let $K$ be the fixpoint of this sequence.
Since $M$ is a minimal model of $P^M$ by definition of FLP answer set, we can prove the claim by showing (i) $K \models P^M$ and (ii) $K \subseteq M$.

(i) Consider a rule $r \in P^M$ such that $K \models B(r)$.
We have to show $H(r) \in K$.
Since $r \in P^M$, $M \models B(r)$ holds.
Thus, $(K,M) \models B(r)$ and therefore $H(r) \in K$.

(ii) We prove $K_i \subseteq M$ for each $i \geq 0$.
We use induction on $i$.
The base case is trivially true as $K_0 = \emptyset \subseteq M$.
Suppose $K_i \subseteq M$ for some $i \geq 0$ in order to prove $K_{i+1} \subseteq M$.
By definition of $K^P_M$, for each $a \in K_{i+1}$ there is $r \in P^M$ such that $H(r) = a$ and $(K_i,M) \models B(r)$.
Thus, $M \models B(r)$, which implies $a \in M$.
\qed
\end{proof}

Therefore, programs with convex structures form a class of programs for which the FLP and PSP semantics coincide. In the following, we will show that it is likely also the largest class for which this holds.

\subsection{Consonance: Complexity of Unrestricted Structures} \label{sec:consonance}

In this section we will examine the computational impact of allowing non-convex structures. We will limit ourselves to structures for which the truth value with respect to an interpretation can be determined in polynomial time. Moreover, we will focus on cautious reasoning, but similar considerations apply also to related problems such as brave reasoning, answer set existence, or answer set checking.

It is known that cautious reasoning over programs with
arbitrary structures under the FLP semantics is $\Pi^P_2$-complete in
general, as shown in \cite{fabe-etal-2011-aij}. Pelov has shown $\Sigma^P_2$-completeness for deciding the existence of PSP answer sets in \cite{pelo-2004}, from which $\Pi^P_2$-completeness for cautious reasoning under the PSP semantics can be derived. We formally state this result now and provide a different proof than Pelov's that will more directly lead to the subsequent considerations.

\begin{theorem}\label{thm:PSP-completeness}
Cautious reasoning under PSP semantics is $\Pi^P_2$-complete.
\end{theorem}
\begin{proof}
Membership follows by Corollary~1 of \cite{son-pont-2007-tplp}.
For the hardness, we provide a reduction from 2-QBF$_\forall$.
Let $\Psi = \forall x_1 \cdots \forall x_m \exists y_1 \cdots \exists y_n\ E$,
where $E$ is in 3CNF.
Formula $\Psi$ is equivalent to $\neg \Psi'$, where $\Psi' = \exists x_1 \cdots \exists x_m \forall y_1 \cdots \forall y_n\ E'$, and $E'$ is a 3DNF equivalent to $\neg E$ and obtained by applying De Morgan's laws.
To prove the claim we construct a program $P_{\Psi}$ such that $P_{\Psi} \models^\mathit{PSP}_c w$ ($w$ a fresh atom) if and only if $\Psi$ is valid, i.e., iff $\Psi'$ is invalid.

Let $E' = (l_{1,1} \wedge l_{1,2} \wedge l_{1,3}) \vee \cdots \vee (l_{k,1} \wedge l_{k,2} \wedge l_{k,3})$, for some $k \geq 1$.
Program $P_{\Psi}$ is the following:
$$
\begin{array}{lllc}
x_{i}^T \leftarrow not\ x_{i}^F \quad & x_{i}^F \leftarrow not\ x_{i}^T \quad & i \in \{1,\ldots,m\} \qquad & \equationlabel{eq:guessX}\\
y_{i}^T \leftarrow not\ y_{i}^F & y_{i}^F \leftarrow not\ y_{i}^T & i \in \{1,\ldots,n\} & \equationlabel{eq:guessY} \\
y_{i}^T \leftarrow \mathit{sat} & y_{i}^F \leftarrow \mathit{sat} & i \in \{1,\ldots,n\} & \equationlabel{eq:saturate} \\
\mathit{sat} \leftarrow \mu(E') &&& \equationlabel{eq:sat} \\
w \leftarrow not\ \mathit{sat} && & \equationlabel{eq:unsat}
\end{array}
$$
where $\mu$ is defined recursively as follows:
\begin{itemize}
\item
$\mu(E') := (\mu(l_{1,1}) \wedge \mu(l_{1,2}) \wedge \mu(l_{1,3})) \vee \cdots \vee (\mu(l_{k,1}) \wedge \mu(l_{k,2}) \wedge \mu(l_{k,3}))$;
\item
$\mu(x_i) := x_{i}^T$ and $\mu(\neg~x_i) := x_{i}^F$ for all $i = 1, \ldots, m$;
\item
$\mu(y_i) := y_{i}^T$ and $\mu(\neg~y_i) := y_{i}^F$ for all $i = 1, \ldots, n$.
\end{itemize}
Note that structure $\mu(E')$ can also be encoded by means of a \emph{sum} aggregate as shown in \cite{alvi-etal-2011-jair-aggregates}.

Rules (\ref{eq:guessX})--(\ref{eq:guessY}) force each PSP answer set of $P_{\Psi}$ to contain at least one of $x_{i}^T$, $x_{i}^F$ ($i \in \{1,\ldots,m\}$), and one of $y_{j}^T$, $y_{j}^F$ ($j \in \{1,\ldots,m\}$), respectively, encoding an assignment of the propositional variables in $\Psi'$.
Rules (\ref{eq:saturate}) are used to simulate universality of the $y$ variables, as described later.
Having an assignment, rule (\ref{eq:sat}) derives $\mathit{sat}$ if the assignment satisfies some disjunct of $E'$ (and hence also $E'$ itself).
Finally, rule (\ref{eq:unsat}) derives $w$ if $\mathit{sat}$ is false.

We first show that $\Psi$ not valid implies $P_{\Psi} \not\models^\mathit{PSP}_c w$.
If $\Psi$ is not valid, $\Psi'$ is valid.
Hence, there is an assignment $\nu$ for $x_1, \ldots, x_m$ such that no extension to $y_1, \ldots, y_n$ satisfies $E$, i.e., all these extensions satisfy $E'$.
Let us consider the following interpretation (which is also a model of $P_\Psi$):
$$
\begin{array}{lll}
M &=& \{x_{i}^T \mid \nu(x_i) = 1,\ i = 1, \ldots, m\} \cup  \{x_{i}^F \mid \nu(x_i) = 0,\ i = 1, \ldots, m\} \\
 &\cup& \{y_{i}^T, y_{i}^F \mid i = 1, \ldots, n\} \cup \{\mathit{sat}\}
\end{array}
$$
We claim that $M$ is a PSP answer set of $P_\Psi$.
In fact, $K^{P_\Psi}_M(\emptyset) \supseteq \{x_{i}^T \mid \nu(x_i) = 1,\ i = 1, \ldots, m\} \cup  \{x_{i}^F \mid \nu(x_i) = 0,\ i = 1, \ldots, m\}$ because of rules (\ref{eq:guessX}) in $P_\Psi^M$.
Since any assignment for the $y$s satisfies at least a disjunct of $E'$, from rule (\ref{eq:sat}) we derive $\mathit{sat} \in K^{P_\Psi}_M(K^{P_\Psi}_M(\emptyset))$.
Hence, rules (\ref{eq:saturate}) force all $y$ atoms to belong to $K^{P_\Psi}_M(K^{P_\Psi}_M(K^{P_\Psi}_M(\emptyset)))$, which is thus the least fixpoint of $K^{P_\Psi}_M$ and coincides with $M$.

Now we show that $P_{\Psi} \not\models^\mathit{PSP}_c w$ implies that $\Psi$ is not valid.
To this end, let $M$ be a PSP answer set of $P_{\Psi}$ such that $w \notin M$.
Hence, by rule (\ref{eq:unsat}) we have that $M \models \mathit{sat}$.
From $\mathit{sat} \in M$ and rules (\ref{eq:saturate}), we have $y_{i}^T,y_{i}^F \in M$ for all $i = 1, \ldots, n$.
And $M$ contains either $x_{i}^T$ or $x_{i}^F$ for $i = 1, \ldots, m$ because of rules (\ref{eq:guessX}).
Suppose by contradiction that $\Psi$ is valid.
Thus, for all assignments of $x_1, \ldots, x_m$, there is an assignment for $y_1, \ldots, y_n$ such that $E$ is true, i.e., $E'$ is false.
We can show that the least fixpoint of $K^{P_\Psi}_M$ is
$K^{P_\Psi}_M(\emptyset) = \{x_{i}^T \mid \nu(x_i) = 1,\ i = 1, \ldots, m\} \cup  \{x_{i}^F \mid \nu(x_i) = 0,\ i = 1, \ldots, m\}$.
In fact, $\mathit{sat}$ cannot be derived because $K^{P_\Psi}_M(\emptyset) \not\models \mu(E')$. We thus have a contradiction with the assumption that $M$ is a PSP answer set of $P_\Psi$.
\qed
\end{proof}

It is also known that the complexity drops to \coNP if
structures in body rules are constrained to be convex. This appears to
be ``folklore'' knowledge and can be argued to follow from results in
\cite{liu-trus-2006-jair}. An easy way to see membership in \coNP is
that all convex structures can be decomposed into a conjunction of a
monotonic and an antimonotonic structure, for which membership in \coNP
has been shown in \cite{fabe-etal-2011-aij}.

It is instructive to note a crucial difference between the $\Pi^P_2$-hardness proofs in \cite{fabe-etal-2011-aij} (and a similar one in \cite{ferr-2005-lpnmr}) and the proofs for Theorem~\ref{thm:PSP-completeness} and the $\Sigma^P_2$ result for PSP in \cite{pelo-2004}.

The fundamental tool in the FLP hardness proofs is the availability of structures $S_1,S_2$ that
allow for encoding ``need to have either atom $x^T$ or $x^F$, or both
of them, but the latter only upon forcing the truth of both atoms.''
$S_1, S_2$ have domains $D_{S_1} = D_{S_2} = \{x^T,x^F\}$ and the following satisfaction patterns:
\[
\begin{array}{llll}
\emptyset \models S_1 \qquad & \{x^T\} \models S_1 \qquad & \{x^F\} \not\models S_1 \qquad & \{x^T,x^F\} \models S_1 \\
\emptyset \models S_2 & \{x^T\} \not\models S_2 & \{x^F\} \models S_2 & \{x^T,x^F\} \models S_2 \\
\end{array}
\]
The reductions then use these structures in a similar way than disjunction is used in the classic $\Sigma^P_2$-hardness proofs in \cite{eite-gott-95}. In particular, the same structures are used for all instances to be reduced.

On the other hand, in the PSP hardness proofs, one dedicated structure is used for each instance of the problem reduced from (2QBF in Theorem~\ref{thm:PSP-completeness}). Indeed, a construction using structures $S_1,S_2$ as described earlier is not feasible for PSP, because $(\emptyset,\{x^T,x^F\}) \not \models S_1$ and $(\emptyset,\{x^T,x^F\}) \not \models S_2$. This is because there is one satisfaction ``hole'' between $\emptyset$ and $\{x^T,x^F\}$ for both $S_1$ and $S_2$.
In the next section, we will exploit this difference.

\subsection{Dissonance: Complexity of Non-convex Structures with Bounded Domains} \label{sec:dissonance}

In this section, we look more carefully at programs with non-convex structures and identify computational differences between the FLP and PSP semantics.
In \cite{alvi-fabe-2013-lpnmr} it has been shown that any non-convex structure (plus all of its variants) can be used in order to implement $S_1$ and $S_2$.
 This result makes it clear that the presence of any non-convex structure that is closed under variants causes a complexity increase for the FLP semantics (unless the polynomial hierarchy collapses). From the above considerations, it is immediately clear that the same construction is not feasible for PSP. It turns out that also no alternative way exists to obtain a similar result, and that the difference in the $\Pi^P_2$-hardness proofs for FLP and PSP is intrinsic.

We start by considering a simple non-convex structure $\Angstrom$ with $D_{\Angstrom}=\{x,y\}$ and $I \models \Angstrom$ if and only if $|I\cap D_{\Angstrom}| \neq 1$. Therefore, $\Angstrom$ behaves like a cardinality constraint $\mathit{COUNT}(\{x,y\})\neq 1$.

\begin{proposition}\label{prop:omega-check}
Deciding whether an interpretation $M$ is a PSP answer set of a datalog$^{\{\Angstrom\}}$ program $P$ is feasible in polynomial time, in particular $DTIME(m^2)$, where $m$ is the number of rules in $P$.
\end{proposition}
\begin{proof}
For any interpretation, testing whether $(I,M) \models \Angstrom\sigma$ (for a variant $\Angstrom\sigma$ of $\Angstrom$) can be done by examining  $|I\cap D_{\Angstrom\sigma}|= i$ and $|M\cap D_{\Angstrom\sigma}| = j$ and returning false if either one of $i$, $j$ is 1, or if $i=0$ and $j=2$. Alternatively, in a less syntax dependent way, one can test whether $M \models \Angstrom\sigma$ and $(I\cup J) \models \Angstrom\sigma$ for each $J \subseteq (M\cap D_{\Angstrom\sigma})\setminus (I\cap D_{\Angstrom\sigma})$. Since there are at most 4 different $J$ for each $I$, either method is feasible in constant time.

For determining whether $M$ is a PSP answer set of $P$, we can check whether it is the least fixpoint of $K^P_M$. Computing the least fixpoint takes at most $m$ applications of $K^P_M$ (where $m$ is the number of rules in $P$). Each application of $K^P_M$ involves in turn at most $m$ tests for $(I,M) \models \Angstrom\sigma$.
\qed
\end{proof}

Given Proposition~\ref{prop:omega-check} it follows that cautious reasoning is still in $\coNP$ for datalog$^{\{\Angstrom\}}$ programs under the PSP semantics.

\begin{proposition}\label{prop:omega-cautious}
Given a datalog$^{\{\Angstrom\}}$ program $P$ and an atom $a$, deciding $P \models^{\mathit{PSP}}_c a$ is in $\coNP$.
\end{proposition}
\begin{proof}
The complement has an immediate nondeterministic polynomial time algorithm: guess an interpretation $M$ and verify in polynomial time that $a \not\in M$ and that $M$ is a PSP answer set of $P$ (by virtue of Proposition~\ref{prop:omega-check}).
\qed
\end{proof}

It follows that for datalog$^{\{\Angstrom\}}$ cautious reasoning (and also answer set existence and brave reasoning) is more complex for the FLP semantics than for the PSP semantics (unless the polynomial hierarchy collapses to its first level).

Examining this result and its proof carefully, we can see that it depends on the fact that each $D_{\Angstrom\sigma}$ contains 2 elements and therefore at most 4 satisfaction tests are needed to determine $(I,M) \models \Angstrom\sigma$. Indeed, we can apply similar reasoning whenever the domains of involved structures are smaller than a given bound.

\begin{theorem}\label{thm:fixedParameterTractability}
Let $P$ be a program. 
If $k$ is an upper bound for the domain size of any structure occurring in $P$,
then checking whether a given interpretation $M$ is a PSP answer set of $P$ is decidable in $DTIME\nop{^C}(2^k m^2 p(n))$, where $m$ is the number of rules in $P$ and $p(n)$ is the polynomial function (in terms of the input size $n$) bounding determining satisfaction of any aggregate in $P$.
\end{theorem}
\begin{proof}
We show that the least fixpoint of $K^P_M$ can be computed in time $O(2^k m^2 p(n))$. 
In the worst case, each application of the operator derives at most one new atom,
and thus the fixpoint is reached after at most $m$ applications of the operator.
Each application requires at most the evaluation of all rules of $P$, and thus at most $m^2$ rule evaluations are sufficient.
To evaluate a rule, the truth of the body has to be checked w.r.t.\ at most $2^k$ interpretations (similar to Proposition~\ref{prop:omega-check}, in which $k=2$), each requiring $p(n)$ time.
We thus obtain the bound $O(2^k m^2 p(n))$.
\qed
\end{proof}

This means that actually most languages with non-convex structures exhibit a complexity gap between the FLP and PSP semantics. There is a uniformity issue here, which we informally noted earlier when examining the $\Pi^P_2$-hardness proof for cautious reasoning under PSP. We can now formalize this, as it follows from Theorem~\ref{thm:fixedParameterTractability} that we need an infinite number of inherently different non-convex structures in order to obtain $\Pi^P_2$ hardness.

\begin{corollary}
Let $\mathbf{S}$ be any finite set of structures, possibly including non-convex structures.
Cautious reasoning over datalog$^\mathbf{S}$ is in \coNP under the PSP semantics.
\end{corollary}

This means that there is also a clear difference in uniformity between the complexity boundary of the FLP and the PSP semantics, respectively. It also means that it is impossible to simulate the FLP semantics in a compact way using the PSP semantics on the class of programs with bounded domain structures, unless the polynomial hierarchy collapses to its first level. The general picture of our complexity results is shown in Figure~\ref{fig:complexity}. We can see that the complexity transition from $\coNP$ to $\Pi^P_2$ is different for the FLP and PSP semantics, respectively. The solid line between convex and non-convex structures denotes a crisp transition for FLP, while the dashed line between bounded non-convex and unbounded non-convex structures is a rougher transition.

\begin{figure}[t]
\centering
\tikzstyle{none} = [text centered]
\tikzstyle{full} = [draw, thick, rectangle]
\tikzstyle{bounded} = [draw, thick, dashed, ellipse]
\tikzstyle{convex} = [draw, thick, ellipse]

\definecolor{FLPcolor}{rgb}{.7,0,0}
\definecolor{PSPcolor}{rgb}{0,0,.7}

\begin{tikzpicture}
\node at (-1,.5) [none](convexL) {\textcolor{FLPcolor}{\coNP}};
\node at (1,.5) [none](convexR) {\textcolor{PSPcolor}{\coNP}};
\node at (0,.25) [none](convexT) {\emph{Convex}};
\node [convex, fit = (convexL) (convexR) (convexT)](convex) {};

\node at (-2,2) [none](boundedL) {\textcolor{FLPcolor}{$\Pi^P_2$}};
\node at (2,2) [none](boundedR) {\textcolor{PSPcolor}{\coNP}};
\node at (0,1.75) [none](boundedT) {\emph{Bounded domain}};
\node [bounded, fit = (boundedL) (boundedR) (boundedT) (convex)](bounded) {};

\node at (-3,3) [none](fullL) {\textcolor{FLPcolor}{$\Pi^P_2$}};
\node at (3,3) [none](fullR) {\textcolor{PSPcolor}{$\Pi^P_2$}};
\node[full] [fit = (fullL) (fullR) (bounded)] {};

\node at (-2,3.75) [none] {\textcolor{FLPcolor}{\textbf{FLP semantics}}};
\node at (2,3.75) [none] {\textcolor{PSPcolor}{\textbf{PSP semantics}}};

\end{tikzpicture}
\caption{Complexity of cautious reasoning}
\label{fig:complexity}
\end{figure}
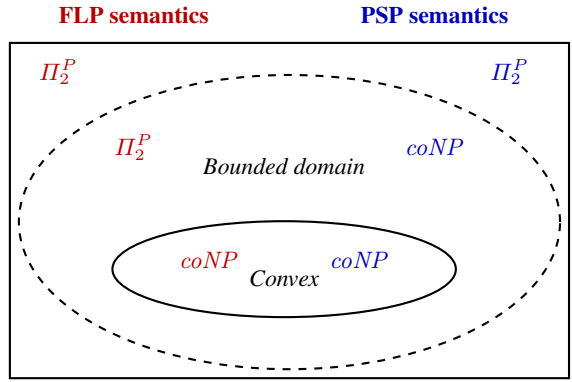

\section{Discussion} \label{sec:discussion}

Looking at Figure~\ref{fig:complexity}, the transition from $\coNP$ to
$\Pi^P_2$ appears somewhat irregular for PSP, as the availability of
single non-convex structures does not cause the transition, but only
their union. However, in practice the availability of an infinite
number of different structures is not unusual: indeed, if aggregates are considered, the presence of one aggregate function and suitable comparison relations usually gives rise to such an infinite repertoire of structures.

\begin{example}
Consider the availability of $\mathit{COUNT}$ over any set of atoms and the comparison relation $\neq$. The structures generated by aggregates of the form $\mathit{COUNT}(S) \neq i$ do not have a bound on the domains of non-convex aggregates. Indeed, for any structure $\mathit{COUNT}(\{a_1, \ldots, a_k\}) \neq 1$, which is non-convex and for which the domain size is $k$, one can formulate also $\mathit{COUNT}(\{a_1, \ldots, a_{k+1}\}) \neq 1$, which is also non-convex and has a larger domain.
\end{example}

However, as noted earlier, for expressing $\Pi^P_2$-hard problems, one
needs a non-uniform approach for PSP, in the sense that a dedicated
aggregate has to be formulated for each problem instance, whereas for
FLP one can re-use the same aggregates for all problem instances.

In practical terms, our results imply that for programs containing only convex
structures, techniques as those presented in
\cite{alvi-etal-2011-jair-aggregates} for FLP can be used for computing answer
sets also for PSP, and techniques presented for PSP can be used for FLP in turn.
It also means that this is the largest class for which this can be done with currently available methods in an efficient way.
There are several examples for convex structures that are
easy to identify syntactically: count aggregates with equality guards,
sum aggregates with positive summands and equality guards, dl-atoms
that do not involve $\capminus$ and rely on a tractable Description
Logic \cite{eite-etal-2008-AIJ}. However many others are in general not
convex, for example sum aggregates that involve both positive and negative
summands, times aggregates that involve the factor 0, average
aggregates, dl-atoms with $\capminus$, and so on. It is still possible
to find special cases of such structures that are convex, but that
 requires deeper analyses.

The results also immediately imply impossibility and possibility results for
rewritability: unless the polynomial hierarchy collapses to its first
level, it is not possible in the FLP semantics to rewrite a program with non-convex
structures into one containing only convex structures (for example, a
program not containing any generalized atoms), unless disjunction or
similar constructs are allowed in rule heads. On the other hand, such rewritings are possible for the PSP semantics if the non-convex structures are guaranteed to have bounded domains. 
This seems to be most
important for dl-programs, where such rewritings are sought after.

The semantics considered in this paper encompass several approaches suggested for programs that couple answer set programming with description logics. The approaches presented in \cite{eite-etal-2005-ijcai} and \cite{luka-2010-tkde} directly employ the FLP semantics, while the approach of \cite{shen-wang-2012-aaai} is shown to be equivalent to the PSP semantics. There are other proposals, such as  \cite{eite-etal-2008-AIJ}, which appears to be different from both FLP and PSP already on convex structure. In future work we plan to relate also these other semantics with FLP and PSP and attempt to identify the largest coinciding classes of programs.

\end{document}